\documentclass[runningheads]{llncs}

\usepackage{eccv}

\usepackage{marvosym}

\usepackage{eccvabbrv}

\usepackage{graphicx}
\usepackage{booktabs}
\usepackage[table]{xcolor}
\usepackage[accsupp]{axessibility}  

\usepackage{hyperref}

\usepackage{orcidlink}

\newcommand{\blfootnote}[1]{%
	\begingroup
	\renewcommand\thefootnote{}
	\footnotetext{#1}%
	\endgroup
}

\begin{document}

\title{
Clinical Cognition Alignment for Gastrointestinal Diagnosis with Multimodal LLMs
}

\titlerunning{CogAlign}

\author{
Huan Zheng$^{1,*}$,Yucheng Zhou$^{1,*}$, Tianyi Yan$^1$, Dubing Chen$^1$, \\ Hongbo Lu$^{2,3}$, Wenlong Liao$^2$, Tao He$^2$, Pai Peng$^2$, Jianbing Shen$^{1,}$$^{\text{\Letter}}$
}

\authorrunning{H. Zheng et al.}

\institute{
$^1$SKL-IOTSC, CIS, University of Macau \\
$^2$Shanghai Jiao Tong University \\
$^3$COWARobot Co. Ltd.
}

\maketitle

\begin{abstract}
Multimodal Large Language Models (MLLMs) have demonstrated remarkable potential in medical image analysis. 
However, their application in gastrointestinal endoscopy is currently hindered by two critical limitations: the misalignment between general model reasoning and standardized clinical cognitive pathways, and the lack of causal association between visual features and diagnostic outcomes. 
In this paper, we propose a novel Clinical-Cognitive-Aligned (CogAlign) framework to address these challenges. 
First, we endow the model with rigorous clinical analytical capabilities by constructing the hierarchical clinical cognition dataset and employing Supervised Fine-Tuning (SFT). 
Unlike conventional approaches, this strategy internalizes the hierarchical diagnostic logic of experts, ranging from anatomical localization and morphological evaluation to microvascular analysis, directly into the model. 
Second, to eliminate visual bias, we provide a theoretical analysis demonstrating that standard supervised tuning inevitably converges to spurious background correlations. 
Guided by this insight, we propose a counterfactual-driven reinforcement learning strategy to enforce causal rectification. 
By generating counterfactual normal samples via lesion masking and optimizing through clinical-cognition-centric rewards, we constrain the model to strictly ground its diagnosis in causal lesion features. 
Extensive experiments demonstrate that our approach achieves State-of-the-Art (SoTA) performance across multiple benchmarks, significantly enhancing diagnostic accuracy in complex clinical scenarios.

\end{abstract}

\blfootnote{$*$Equal contribution. $\text{\Letter}$ Corresponding author: \textit{Jianbing Shen}. }

\section{Introduction}
Gastrointestinal (GI) malignancies constitute a substantial portion of the global cancer burden, establishing endoscopic screening as the gold standard for early detection and intervention \cite{motta2021gastrointestinal}. 
Given the high dependence on operator experience and the inherent inter-observer variability in clinical practice, computer-aided diagnosis systems have emerged as a critical support tool to mitigate miss rates \cite{doi2007computer,yanase2019systematic}. 
Over the past decade, data-driven deep learning approaches, particularly Convolutional Neural Networks (CNNs) \cite{he2016deep} and Vision Transformers (ViTs) \cite{dosovitskiy2020image,vaswani2017attention}, have demonstrated expert-level proficiency in specialized tasks such as polyp detection \cite{soleymanjahi2024artificial} and lesion classification \cite{thieme2023deep,zhao2026agentic}. 
Conventional paradigms are fundamentally restricted by their closed-set nature and opaque decision-making processes. 
These discriminative models typically function as silent classifiers that output rigid categorical labels without providing the underlying diagnostic rationale \cite{zhou2026medical,wang2025survey}. 
Such opacity precludes clinical validation, undermining the diagnostic reliability required for high-stakes medical environments.

The recent advent of MLLMs marks a transformative shift from specialized discriminative models to generalized reasoning agents \cite{shool2025systematic,xu_2026_CVPR,xiao2026reversible}. 
By synergizing the perceptual capabilities of advanced visual encoders with the extensive knowledge and inferential power of Large Language Models (LLMs), MLLMs introduce a versatile framework for endoscopic analysis \cite{liu2025endobench,jiang2025hulu}. 
Unlike their predecessors, these foundation models possess the unique capacity to process visual information and generate coherent linguistic descriptions simultaneously \cite{xu2025lingshu}. 
This paradigm offers the potential to mimic the workflow of an endoscopist by identifying pathological features and providing comprehensive reports \cite{chen2024towards}. 

\begin{figure}[t]
    \centering
    \includegraphics[width=\linewidth]{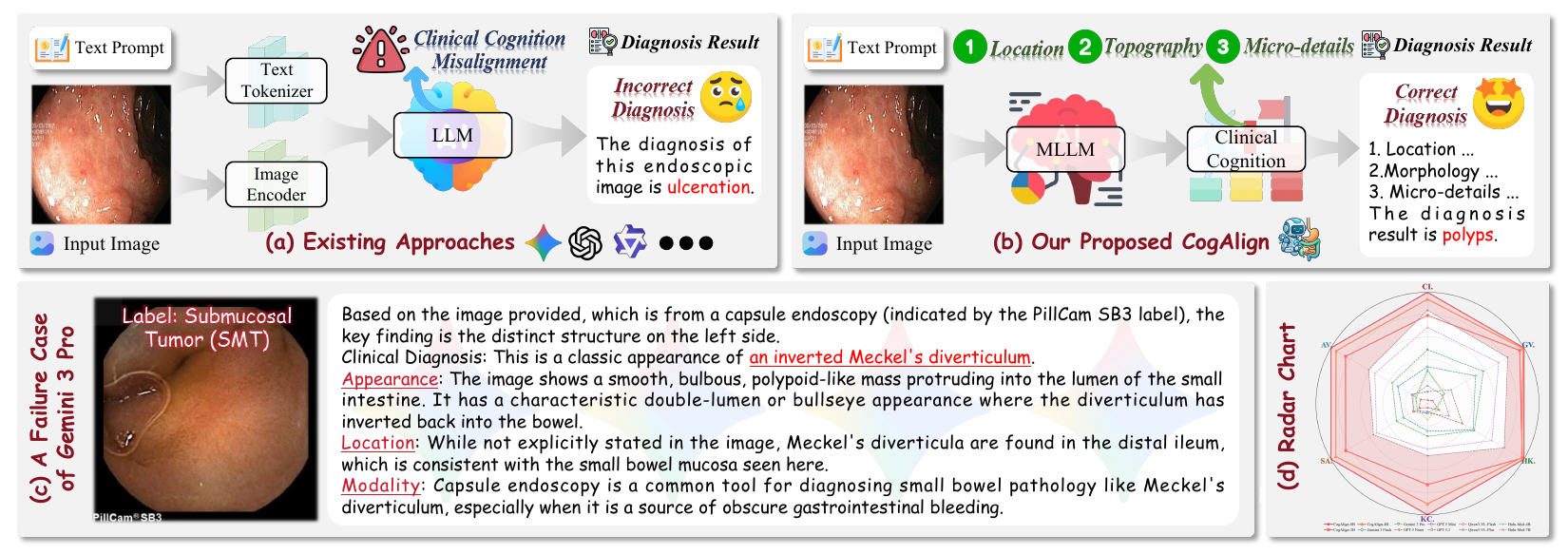}
    \vspace{-7mm}
    \caption{\textbf{Illustration of the motivation.}
  (a) Existing methods suffer from clinical cognition misalignment.
  (b) Our CogAlign framework enforces a strict clinical cognitive flow.
  (c) A representative failure case generated by Gemini 3 Pro.
  (d) A radar chart highlighting the superior accuracy of CogAlign across diverse benchmarks.
    }
    \label{fig:motivation}
    \vspace{-6mm}
\end{figure}

Despite this promise, the direct deployment of general MLLMs \cite{jiang2025hulu,bai2025qwen3} in gastrointestinal endoscopy is hindered by two critical limitations, as illustrated in Fig.~\ref{fig:motivation} (a) and (b). 
The first is the misalignment between general model reasoning and standardized clinical cognitive pathways. 
In clinical practice, an endoscopist's diagnosis follows a rigorous, hierarchical cognitive flow: initially localizing the anatomical site, subsequently evaluating morphological features, analyzing micro-details, and finally concluding with a diagnosis. 
In contrast, general MLLMs often exhibit scattered reasoning, skipping critical analytical steps or hallucinating non-existent features. 
This cognitive gap renders their outputs unreliable for high-stakes medical decisions. 
The second limitation is the lack of causal association between visual features and diagnostic outcomes. 
MLLMs are susceptible to confounding visual factors, frequently relying on spurious correlations in the background, rather than characterizing the pathological lesion itself. 
As shown in the failure case in Fig.~\ref{fig:motivation}(c), even advanced models like Gemini 3 Pro can be misled by environmental artifacts, causing them to hallucinate a diagnosis based on the capsule modality context rather than the actual submucosal tumor features. 
This absence of causal grounding makes the models brittle and prone to failure when deployed in diverse clinical environments where such artifacts vary. 
As shown in Fig.~\ref{fig:motivation}(d), these deficiencies collectively constrain the diagnostic capability of existing models, resulting in suboptimal accuracy.

To address these challenges, we propose CogAlign for gastrointestinal diagnosis. 
Our approach is designed to bridge the gap between general reasoning and expert clinical protocols while ensuring diagnoses are strictly grounded in medical visual features. 
First, to tackle the clinical cognition misalignment, we construct a hierarchical clinical cognition dataset that encapsulates the step-by-step diagnostic logic of experts. 
We internalize this structured assessment process into the model, enforcing a diagnostic trajectory that moves strictly from anatomical localization and morphological evaluation to micro-details analysis. 

Second, to resolve the issue of visual bias, we provide a theoretical analysis demonstrating that standard supervised tuning inevitably converges to spurious background shortcuts. 
Guided by this insight, we introduce a counterfactual-driven Group Relative Policy Optimization (GRPO) strategy for causal rectification. 
By masking lesion areas to generate counterfactual normal samples, we construct a counterfactual reference to isolate lesion-specific features. 
We then optimize the model using clinical-cognition-centric rewards, constraining the outcomes to be causally grounded in visual evidence of the lesion rather than background.
Our contributions are summarized as follows:
\begin{itemize}
    \item We propose CogAlign, a novel framework bridging the gap between general model capabilities and specialized clinical requirements. It integrates hierarchical cognitive tuning with counterfactual-driven reinforcement learning to ensure reliable gastrointestinal diagnosis.
    \item We construct a new dataset and apply SFT to instill rigorous analytical capabilities. This allows the model to emulate expert logic, progressing systematically from anatomical localization to microscopic detail analysis.
    \item We theoretically demonstrate that standard tuning relies on spurious background shortcuts and introduce a counterfactual-driven GRPO strategy to rectify this bias. Using counterfactual normal samples and clinical-cognition-centric rewards, we enforce strict causal grounding in pathological features.
    \item Extensive evaluations confirm that our approach achieves SoTA performance.
\end{itemize}

\section{Related Work}

\subsection{Medical Multimodal Large Language Models}
The rapid evolution of general Multimodal Large Language Models (MLLMs) has sparked significant interest in adapting these foundation models for the medical domain \cite{chen2025shizhengpt,zhou2026comem,zhang2026pi,mullappilly2024bimedix2,zhou2025gsq,zhu2025pathology,zhu2026medeyes}. 
By aligning powerful vision encoders with autoregressive language models, researchers have developed systems capable of interpreting complex clinical imagery and generating coherent text \cite{lin2025healthgpt, moor2023med}. 
Early pioneering models such as LLaVAMed \cite{li2023llava} demonstrated the feasibility of adapting general visual instruction tuning to biomedicine \cite{lai2026med}. 
These systems rely on vast datasets of image and text pairs to achieve proficiency in tasks like medical visual question answering, radiology report generation, and broad clinical reasoning \cite{zhang2024generalist, ning2025unimedvl,sun2025chiron}. 

Recent progress has focused on improving domain specific accuracy through parameter efficient fine tuning techniques \cite{hu2022lora} and specialized medical instruction datasets \cite{pan2025medvlm}.
Researchers have successfully scaled these architectures to handle diverse modalities including X-rays, magnetic resonance imaging, and histopathology slides \cite{wang2024llm,zhou2025improving,zhou2025mam,zhou2024visual}.
Despite these impressive capabilities \cite{alkhaldi2024minigpt}, current medical foundation models frequently struggle with diagnostic reliability in high stakes environments. 
They are prone to visual hallucinations and often act as superficial pattern matchers rather than genuine reasoning agents. 
Furthermore, standard training paradigms fail to enforce structured clinical logic, causing these models to skip critical analytical steps. 
They also exhibit severe vulnerability to visual bias, frequently grounding their textual outputs in spurious background correlations rather than genuine pathological evidence. 
Overcoming these fundamental limitations remains a primary hurdle for deploying multimodal models in reliable clinical assistance.

\subsection{Gastrointestinal Disease Diagnosis}
Computer Aided Diagnosis systems have become an integral component of modern gastroenterology, designed to assist clinicians in mitigating interobserver variability and reducing lesion miss rates during endoscopic screening \cite{ramoni2025artificial,lin2026models,zhu2026medsynapsevbridgingvisualperception}. 
Over the past decade, the field has been dominated by discriminative deep learning paradigms \cite{kroner2021artificial}. 
Convolutional Neural Networks and Vision Transformers have been extensively engineered to tackle specific gastrointestinal tasks, achieving expert level accuracy in polyp detection, anatomical landmark recognition, and ulcer classification \cite{fan2020pranet,roth2024domain}. 
Advanced segmentation and object detection methods have been tailored to address the unique visual challenges of endoscopy, such as varying illumination, diverse organ topologies, and specular reflections \cite{hu2026pranet,soleymanjahi2024artificial}.

However, the clinical utility of these conventional methods is inherently restricted by their closed set nature and opaque decision making processes \cite{azad2024advances}. 
Traditional models function as silent classifiers that output rigid categorical predictions without providing the underlying diagnostic rationale \cite{he2025divgi}. 
To address the need for interpretability, recent literature has begun exploring report generation for endoscopy using vision language frameworks \cite{shu2025fleming,nath2025vila}. 
While these preliminary multimodal approaches can produce descriptive text, they generally treat endoscopic analysis as a standard image captioning problem \cite{deria2026medmo}. 
They fail to reflect the rigorous cognitive workflow of a senior endoscopist, which sequentially progresses from spatial anatomical localization to morphological assessment and finally to microscopic detail analysis \cite{mullappilly2026medix}. 
Consequently, current models lack causal diagnostic grounding and remain highly susceptible to environmental noise such as surgical instrument artifacts and mucosal bubbles. 

\begin{figure}[t]
    \centering
    \includegraphics[width=\linewidth]{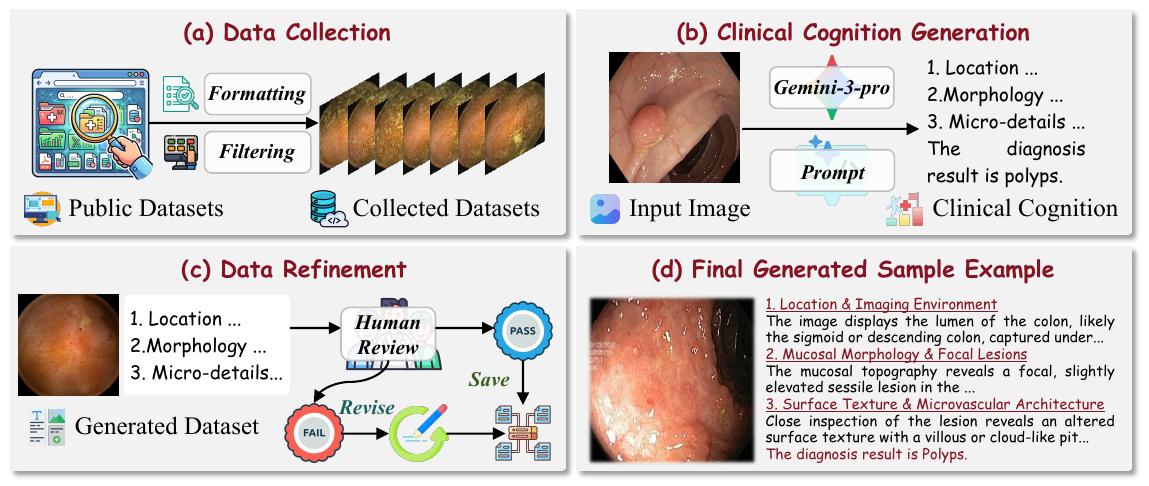}
    \vspace{-7mm}
    \caption{\textbf{Overview of the dataset curation pipeline.} (a) shows the collection and filtering of diverse endoscopic images. (b) shows the generation of hierarchical clinical cognition reasoning chains. (c) shows the human expert refinement process to eliminate hallucinations. (d) shows a generated sample example.}
    \label{fig:dataset}
    \vspace{-6mm}
\end{figure}

\section{Hierarchical Clinical Cognition Dataset}
\label{sec:dataset}
Current public datasets for gastrointestinal endoscopy primarily consist of image-label pairs, lacking the intermediate reasoning steps required for transparent diagnosis \cite{vallee2020crohnipi,jha2023gastrovision,borgli2020hyperkvasir}. 
Training on such data encourages models to learn shortcut features rather than clinical logic. 
Hence, we construct a hierarchical clinical cognition dataset designed to instill expert-level cognitive patterns into the MLLM.

\subsection{Clinical Cognitive Hierarchy Definition}
We define a standardized diagnostic protocol derived from the cognitive workflows of expert gastroenterologists.
Unlike general image captioning, our annotation schema enforces a strict coarse to fine reasoning flow comprising three distinct stages prior to the final diagnosis. This hierarchical structure accurately mirrors the cognitive process of medical experts:
\begin{enumerate}
    \item \textbf{Anatomical Localization:} Identification of the specific organ segment to provide essential spatial context and document the imaging conditions.
    \item \textbf{Morphological Evaluation:} Assessment of macroscopic features, encompassing lesion shape, elevation, size, color, and boundaries.
    \item \textbf{Micro-detail Analysis:} Scrutiny of fine grained surface patterns, such as villous structures, alongside vascular configurations.
\end{enumerate}

\subsection{Human-in-the-Loop Curation Pipeline}
Manually annotating reasoning chains for large scale medical data is prohibitively expensive and time consuming. 
Therefore, we design a semi-automated curation pipeline incorporating a rigorous human in the loop mechanism.

During the data collection phase shown in Fig.~\ref{fig:dataset}(a), we aggregate diverse endoscopic images from public websites. 
A dedicated filtering process ensures the diversity and quality. 
Then, in the clinical cognition generation phase depicted in Fig.~\ref{fig:dataset}(b), we leverage Gemini 3 Pro \cite{gemini3_report}, to act as a teacher model. 
By utilizing a specific prompt that explicitly outlines the three stage hierarchy defined above, we query the teacher model to generate structured reasoning descriptions for each input image. 
Finally, to eliminate potential hallucinations inherent to general multimodal models, we implement a data refinement phase detailed in Fig.~\ref{fig:dataset}(c).
Human experts meticulously review the generated annotations. 
Annotations that pass the review are saved automatically, whereas samples containing factual errors fail the initial inspection and undergo manual revision by the experts. 

\subsection{Dataset Overview}
We construct a comprehensive endoscopic dataset from five prominent public repositories, namely CrohnIPI \cite{vallee2020crohnipi}, GastroVision \cite{jha2023gastrovision}, HyperKvasir \cite{borgli2020hyperkvasir}, Kvasir-Capsule \cite{smedsrud2021kvasir}, and The SEE-AI Project \cite{yokote2024small}, we assemble a total corpus of 24,515 samples. 
We establish a stratified split comprising 19,736 samples for training and 4,779 samples for testing.
Specifically, the dataset encompasses 23 distinct single-label categories and 49 complex multi-label pathology combinations.
As demonstrated by the shown example in Fig.~\ref{fig:dataset}(d), this curation process yields a high-quality dataset denoted as $\mathcal{D} = \{(\mathbf{x}_i, \mathbf{q}_i, \mathbf{r}_i, \mathbf{l}_i)\}_{i=1}^N$. 
In this formulation, $\mathbf{x}_i$ represents the image, $q_i$ is the diagnostic query, $\mathbf{r}_i$ signifies the verified hierarchical clinical cognition reasoning chain, and $\mathbf{l}_i$ denotes the diagnostic label.

\section{Methodology}

\subsection{Problem Definition}
As illustrated in Fig.~\ref{fig:framework}, the proposed CogAlign framework is designed to enforce a dual alignment: (1) aligning the model's reasoning process with the standardized hierarchical cognitive pathways of clinical experts, and (2) grounding with causal pathological features rather than spurious background correlations.

Formally, given an image $\mathbf{x}$ and a diagnostic instruction $\mathbf{q}$, our goal is to generate a response $\mathbf{y}$ that not only provides the correct diagnostic label $\mathbf{l}$ but also produces a structured reasoning chain $\mathbf{r}$ that mirrors clinical standards:
\begin{align}
\mathbf{y} = \mathbf{r} \oplus \mathbf{l} = \mathop{\arg\max}_{\mathbf{y}} P(\mathbf{y} | \mathbf{x}, \mathbf{q}; \theta),
\end{align}
where $\theta$ represents the trainable parameters of the MLLM, and $\oplus$ denotes the sequential concatenation of the reasoning process and the conclusion.

\begin{figure}[t]
    \centering
    \includegraphics[width=0.99\linewidth]{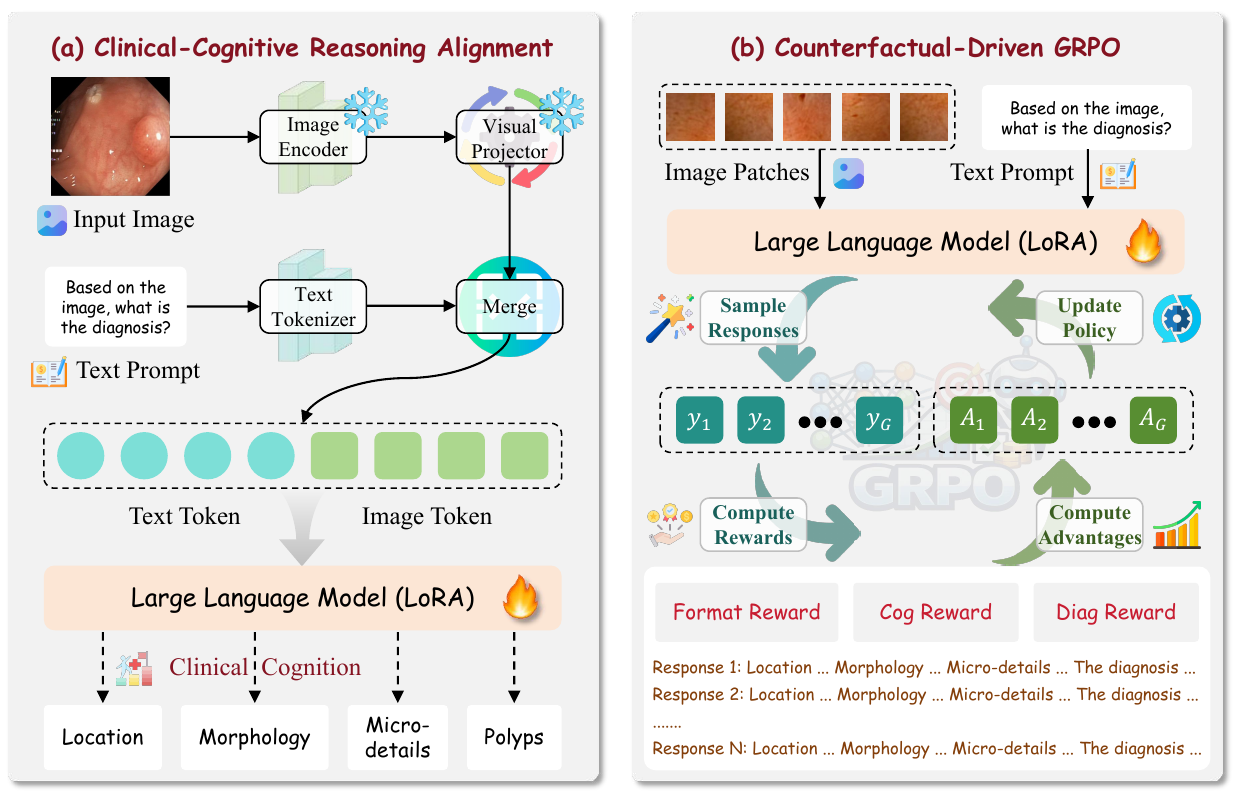}
    \vspace{-3mm}
    \caption{\textbf{Overview of the proposed CogAlign framework.} The pipeline consists of two fundamental stages. Left panel demonstrates the clinical cognitive reasoning alignment phase, where the multimodal large language model undergoes supervised fine tuning.   
    Right panel details the reinforcement learning phase guided by counterfactuals.}
    \label{fig:framework}
    \vspace{-6mm}
\end{figure}

\subsection{Clinical-Cognitive Reasoning Alignment}

General MLLMs \cite{bai2025qwen3,gemini3_report}, while possessing broad semantic knowledge, operate within an unconstrained generative space that often diverges from the disciplined sequential logic of expert endoscopists. 
To bridge this gap, we implement a Clinical-Cognitive Reasoning Alignment phase via SFT. 
The primary objective of this stage is to constrain the model's generation manifold, forcing it to internalize the hierarchical reasoning chain $\mathbf{r}$, from anatomical localization to micro-detail analysis, before yielding a final diagnosis.

Formally, we utilize the hierarchical dataset $\mathcal{D} = \{(\mathbf{x}_i, \mathbf{q}_i, \mathbf{y}_i)\}_{i=1}^N$ constructed in Sec.\ref{sec:dataset}, where $\mathbf{y}_i = \mathbf{r}_i \oplus \mathbf{l}_i$ represents the target sequence concatenating the reasoning steps and the diagnostic conclusion. 
We employ a visual encoder to extract feature embeddings from the endoscopic image $\mathbf{x}_i$, which are projected into the LLM's embedding space. 
The model is then optimized to generate the target sequence $\mathbf{y}_i$ in an autoregressive manner, effectively modeling the joint probability of the reasoning rationale and the diagnostic outcome.
The optimization objective is defined as minimizing the negative log-likelihood of the next token:
\begin{align}
\mathcal{L}_{\text{SFT}}(\theta) = - \frac{1}{N} \sum_{i=1}^{N} \sum_{t=1}^{L_i} \log P(\mathbf{y}_{i,t} | \mathbf{x}_i, \mathbf{q}_i, \mathbf{y}_{i,<t}; \theta),
\end{align}
where $L_i$ denotes the length of the sequence $\mathbf{y}_i$, and $\theta$ represents the trainable parameters. 
Crucially, this objective enforces a strong statistical dependency: the final diagnosis $\mathbf{l}$ becomes a conditional consequence of the preceding morphological and micro-detail analysis contained within $\mathbf{r}$, rather than a direct, opaque classification from visual features.

\subsection{Visual-Cognitive Misalignment and Causal Rectification}
\label{sec:theory}

We provide a formal derivation of why SFT converges to a biased shortcut and how counterfactual intervention mathematically enforces causal grounding.

\begin{definition}[Latent Factor Model]
An image $X$ is generated by $\Psi: \mathcal{Z}_c \times \mathcal{Z}_e \rightarrow \mathcal{X}$, where $Z_c$ and $Z_e$ are causal and spurious latents. The diagnostic model is $f_{\theta}(X) = \sigma( \mathbf{w}^\top \phi(X) )$, where $\phi(X) = [\phi_c(Z_c); \phi_e(Z_e)]$ is the feature.
\end{definition}

\begin{definition}[Effective Feature Sensitivity]
The sensitivity of $f$ to factor $Z_i$ is defined as the norm of the Jacobian:
\begin{align}
    \mathcal{S}_i = \| \nabla_{Z_i} f_{\theta}(\Psi(Z_c, Z_e)) \|_2, \quad i \in \{c, e\}.
\end{align}
\end{definition}

\begin{theorem}[Shortcut Convergence in SFT]
Let $K(Z_e) < K(Z_c)$. Under gradient descent optimization of the SFT loss $\mathcal{L}$, the model parameters $\mathbf{w} = [\mathbf{w}_c; \mathbf{w}_e]$ satisfy $\| \mathbf{w}_e \| > \| \mathbf{w}_c \|$, leading to $\mathcal{S}_e > \mathcal{S}_c$.
\end{theorem}

\begin{proof}
Consider the gradient flow of the SFT objective $\mathcal{L} = -\mathbb{E}[Y \log f + (1-Y) \log (1-f)]$. The dynamics of the weights for each feature are:
\begin{align}
    \frac{d \mathbf{w}_c}{dt} &= -\eta \frac{\partial \mathcal{L}}{\partial \mathbf{w}_c} = \eta \mathbb{E} \left[ (Y - f) \cdot \phi_c(Z_c) \right] \\
    \frac{d \mathbf{w}_e}{dt} &= -\eta \frac{\partial \mathcal{L}}{\partial \mathbf{w}_e} = \eta \mathbb{E} \left[ (Y - f) \cdot \phi_e(Z_e) \right] 
\end{align}
According to the Simplicity Bias principle \cite{shah2020pitfalls}, for low-complexity features $Z_e$, the spectral norm of the corresponding feature mapping $\phi_e$ is larger and converges faster in the early stages of gradient descent:
\begin{align}
    \| \phi_e(Z_e) \| \gg \| \phi_c(Z_c) \| \implies \left\| \frac{d \mathbf{w}_e}{dt} \right\| > \left\| \frac{d \mathbf{w}_c}{dt} \right\|.
\end{align}
As $t \rightarrow \infty$, $(Y - f) \rightarrow 0$. Since $\mathbf{w}_e$ captured the majority of the variance early on, the optimization stagnates before $\mathbf{w}_c$ is fully learned, yielding $\mathcal{S}_e > \mathcal{S}_c$.
\end{proof}

\begin{theorem}[Causal Rectification via Counterfactual Penalty]
Let $\mathcal{R}_{cf} \!=\! \mathbb{E}[ f(\Psi(\mathbf{0}, Z_e))^2 ]$ be the counterfactual penalty. Minimizing the total objective $\mathcal{J} = \mathcal{L} + \lambda \mathcal{R}_{cf}$ as $\lambda \rightarrow \infty$ ensures $\mathcal{S}_e \rightarrow 0$.
\end{theorem}

\begin{proof}
The optimal parameters $\theta^*$ must satisfy the stationary condition $\nabla_{\theta} \mathcal{J} \!=\! 0$:
\begin{align}
    \nabla_{\theta} \mathcal{L} + \lambda \nabla_{\theta} \mathcal{R}_{cf} = 0.
\end{align}
Substituting the gradient of the penalty term $\mathcal{R}_{cf}$:
\begin{align}
    \nabla_{\theta} \mathcal{L} + 2\lambda \mathbb{E} \left[ f(X_{cf}) \cdot \frac{\partial f}{\partial \theta} \right] = 0.
\end{align}
As $\lambda \rightarrow \infty$, for the equation to hold, the model must satisfy $f(X_{cf}) \rightarrow 0$. Given $X_{cf} = \Psi(\mathbf{0}, Z_e)$, this implies:
\begin{align}
    \mathbf{w}_e^\top \phi_e(Z_e) \rightarrow 0, \quad \forall Z_e \in \mathcal{Z}_e.
\end{align}
Consequently, the sensitivity to spurious factors vanishes:
\begin{align}
    \mathcal{S}_e = \left\| \frac{\partial f}{\partial Z_e} \right\| = \left\| \mathbf{w}_e^\top \frac{\partial \phi_e}{\partial Z_e} \right\| \rightarrow 0.
\end{align}
To minimize the remaining $\mathcal{L}$ on the original samples, the model must re-orient its gradient flow toward $\mathbf{w}_c$, maximizing the reliance on causal features $Z_c$.
\end{proof}

\subsection{Counterfactual-Driven GRPO for Causal Alignment}

Guided by the theoretical insights in Sec.\ref{sec:theory}, we introduce a reinforcement learning framework termed counterfactual-driven GRPO. 
This stage operationalizes the counterfactual intervention to explicitly reward the model for grounding its diagnosis in causal lesion features.

\noindent\textbf{Counterfactual Normal Sample Synthesis.}
To eliminate visual bias from background shortcuts, we construct counterfactual samples where lesion features are erased while the environment remains identical.
First, the MLLM generates an initial lesion bounding box, which is rigorously refined by experts to define the precise lesion mask $\mathbb{M}$. 
Second, we apply high-intensity Gaussian smoothing to obliterate diagnostic features within $\mathbb{M}$:
\begin{align}
    \mathbf{x}_{cf} = \mathbf{x} \odot (1 - \mathbb{M}) + \mathcal{G}(\mathbf{x}, \sigma) \odot \mathbb{M}.
\end{align}
Finally, we assign a normal label and a corresponding negative reasoning chain to $\mathbf{x}_{cf}$. 
This paired sample $(\mathbf{x}_{cf}, \mathbf{r}_{cf}, \mathbf{l}_{cf})$ forces the model to ground its diagnosis strictly in lesion features; if it predicts pathology based on the unchanged background in $\mathbf{x}_{cf}$, it incurs a high optimization penalty.

\noindent\textbf{Clinical-Cognition-Centric Rewards.}
To ensure the reasoning chain $\mathbf{r}$ is both structurally compliant and causally grounded, we design several rewards.

\noindent\textit{Output Format Reward.}
To enforce the model adherence to the strict hierarchical structure defined in our clinical cognitive pathway, we design a Format Reward $R_{fmt}$. 
The model's output $\mathbf{y}$ must sequentially cover three critical sections: (1) Location \& Imaging Environment, (2) Mucosal Morphology \& Focal Lesions, and (3) Surface Texture \& Microvascular Architecture.
The reward function is defined as an all-or-nothing constraint:
\begin{align}
    R_{fmt}(y) = \mathbb{I}\left( \bigwedge_{s \in \mathcal{S}} (s \in y) \right),
\end{align}
where $\mathcal{S}$ represents the set of required section headers and $\mathbb{I}(\cdot)$ is the indicator function. 
If any section is missing, the reward is 0; otherwise, it is 1. This forces the model to maintain structural integrity during generation.

\noindent\textit{Clinical Cognition Reward.}
Merely following the correct format is insufficient; the content must capture specific semiological details. 
We propose a Clinical Cognition Reward $R_{cog}$ to enforce semantic precision. 
For each ground truth reasoning chain, we utilize an LLM to pre-extract a set of critical keywords $K_{gt}$, consisting of exactly three key features for each of the three cognitive sections, totaling $|K_{gt}| = 9$ keywords. 
During training, we directly verify the presence of these keywords within the generated response $\mathbf{y}$. 
The reward is calculated as:
\begin{align}
    R_{cog}(\mathbf{y}, K_{gt}) = \frac{1}{9} \sum_{k \in K_{gt}} \mathbb{I}(k \in \mathbf{y}),
\end{align}
where $\mathbb{I}(k \in \mathbf{y})$ is an indicator function that returns 1 if the keyword $k$ appears in the generated text $\mathbf{y}$. 
This mechanism ensures the model explicitly articulates all critical diagnostic criteria across the hierarchy.

\noindent\textit{Diagnostic Consistency Reward.}
The Diagnostic Consistency Reward $R_{diag}$ evaluates the final conclusion extracted from the model's response. Let $\mathbf{l}$ be the diagnosis parsed from $\mathbf{y}$ and $\mathbf{l}_{gt}$ be the ground truth label.
\begin{align}
    R_{diag}(\mathbf{y}, \mathbf{y}_{gt}) = 
    \begin{cases} 
    1, & \text{if } \mathbf{l} = \mathbf{l}_{gt}, \\
    0, & \text{otherwise}.
    \end{cases}
\end{align}
This reward ensures that the reasoning chain culminates in the correct result.

\noindent\textit{GRPO Optimization.}
To align the model with the proposed rewards efficiently, we employ Group Relative Policy Optimization (GRPO), which estimates the baseline directly from the group average of sampled outputs. 
For each input query $\mathbf{q}$, we sample $G$ outputs $\{y_1, \dots, y_G\}$ from the current policy $\pi_{\theta_{old}}$. 
We first compute the total reward $r_i = R_{fmt}(y_i) + \lambda_1 R_{cog}(y_i) + \lambda_2 R_{diag}(y_i)$ for each output $y_i$. To reduce gradient variance, we calculate the normalized group advantage $\hat{A}_i = (r_i - \mu_r)/(\sigma_r + \epsilon)$, where $\mu_r$ and $\sigma_r$ are the mean and standard deviation of the rewards within the sampled group. 
Finally, we optimize the policy $\pi_\theta$ by maximizing the following surrogate objective alongside a KL divergence penalty to prevent deviation from the reference model $\pi_{ref}$:
\begin{equation}
    \mathcal{J}_{GRPO}(\theta) \!=\! \mathbb{E}_{q \sim D} \!\left[ \frac{1}{G} \sum_{i=1}^G \left( \min \left( \rho_i \hat{A}_i, \text{clip}(\rho_i, 1\!-\!\epsilon, 1\!+\!\epsilon) \hat{A}_i \right) \!- \!\beta \mathbb{D}_{KL}(\pi_\theta || \pi_{ref}) \right) \!\right]
\end{equation}
where $\rho_i = \pi_\theta(y_i|q) / \pi_{\theta_{old}}(y_i|q)$ is the probability ratio.

\section{Experiments}

\subsection{Experiment Setup}

\noindent\textbf{Implementation Details.}
We implement the two-stage training using the SWIFT framework with eight NVIDIA L20 GPUs. Stage 1 performs SFT for 400 steps using the AdamW optimizer (learning rate $1 \times 10^{-4}$, cosine scheduler) and a global batch size of 128. The vision encoder and aligner are frozen, while we apply LoRA \cite{hu2022lora} (rank 16, $\alpha=32$) to all linear modules, capping sequence length at 2048 tokens. Stage 2 applies GRPO \cite{guo2025deepseek} for 200 steps. This phase continues LoRA optimization with a global batch size of 256, a reduced learning rate of $1 \times 10^{-6}$, and a KL-divergence penalty $\beta=0.04$. For each query, we sample $G=8$ generations and compute an additive reward weighting format, clinical cognition, and diagnostic consistency at 1.0, 1.0, and 2.0, respectively. 

\begin{table}[t]
\centering
\scriptsize
\setlength{\tabcolsep}{7pt} 
\caption{
\textbf{Quantitative comparison on five gastrointestinal benchmarks.}
We evaluate CogAlign against diverse models.
Abbreviations: CI. (CrohnIPI), GV. (GastroVision), HK. (HyperKvasir), KC. (Kvasir-Capsule), SA. (The SEE-AI Project).
}
\vspace{-3mm}
\label{tab:main_results}
\begin{tabular}{lcccccc}
\toprule
\textbf{Model} & \textbf{CI.} & \textbf{GV.} & \textbf{HK.} & \textbf{KC.} & \textbf{SA.} & \textbf{Average} \\
\midrule
\rowcolor{gray!15} \multicolumn{7}{l}{\textit{\textbf{Large Foundation Models}}} \\
Gemini 3 Flash & 20.87\% & 38.46\% & 43.24\% & 18.32\% & 15.01\% & 20.69\% \\
Gemini 3 Pro & 30.58\% & 44.73\% & 44.40\% & 21.83\% & 19.20\% & 24.82\% \\
GPT-5 Nano & 1.94\% & 3.99\% & 10.81\% & 2.77\% & 5.14\% & 5.06\% \\
GPT-5 Mini & 10.19\% & 11.97\% & 20.46\% & 6.50\% & 9.04\% & 10.04\% \\
GPT-5.2 & 6.80\% & 18.80\% & 33.20\% & 5.32\% & 8.32\% & 11.13\% \\
Qwen3-VL-Flash & 43.20\% & 56.98\% & 61.00\% & 30.35\% & 31.65\% & 36.93\% \\
Qwen3-VL-Plus & 52.91\% & 64.10\% & 72.78\% & 34.72\% & 33.63\% & 41.16\% \\
\midrule
\rowcolor{gray!15} \multicolumn{7}{l}{\textit{\textbf{Medical Foundation Models}}} \\
Hulu-Med-4B & 18.45\% & 13.68\% & 7.92\% & 6.50\% & 6.55\% & 7.72\% \\
Hulu-Med-7B & 19.42\% & 13.39\% & 9.46\% & 10.86\% & 6.22\% & 8.58\% \\
\midrule
\rowcolor{gray!15} \multicolumn{7}{l}{\textit{\textbf{Small Foundation Models}}} \\
Qwen3-VL-2B & 18.93\% & 32.48\% & 33.20\% & 11.71\% & 12.01\% & 16.05\% \\
Qwen3-VL-4B & 36.89\% & 52.99\% & 50.39\% & 22.04\% & 25.50\% & 30.03\% \\
Qwen3-VL-8B & 39.32\% & 47.01\% & 67.57\% & 30.14\% & 29.22\% & 35.30\% \\
\midrule
\rowcolor{gray!15} \multicolumn{7}{l}{\textit{\textbf{SFT on The Proposed Dataset}}} \\
Qwen3-VL-2B (SFT) & 41.26\% & 73.50\% & 87.26\% & 50.16\% & 48.39\% & 54.49\% \\
Qwen3-VL-4B (SFT) & 55.34\% & 76.07\% & 86.10\% & 64.75\% & 55.23\% & 61.98\% \\
Qwen3-VL-8B (SFT) & 62.14\% & 76.92\% & 89.38\% & 72.74\% & 58.77\% & 66.31\% \\
\midrule
\rowcolor{green!15} \multicolumn{7}{l}{\textit{\textbf{Our Proposed Models}}} \\
\rowcolor{green!5}
CogAlign-2B & 50.00\% & 73.79\% & 89.77\% & 53.99\% & 50.96\% & 57.40\%\\
\rowcolor{green!5}
CogAlign-4B & 59.22\% & 76.35\% & 89.19\% & 66.77\% & 57.22\% & 64.05\% \\
\rowcolor{green!5}
CogAlign-8B &  \textbf{63.11\%} & \textbf{77.21\%} & \textbf{91.51\%} & \textbf{74.01\%} & \textbf{60.18\%} & \textbf{67.67\%} \\
\bottomrule
\end{tabular}
\vspace{-6mm}
\end{table}

\noindent\textbf{Baselines.}
We evaluate the performance of CogAlign against a comprehensive suite of SoTA models.
For the large foundation models, we include proprietary systems such as Gemini 3 Flash, Gemini 3 Pro, GPT-5.2, GPT-5 Mini, and GPT-5 Nano. 
We also benchmark against the Qwen3-VL series \cite{bai2025qwen3}, specifically Qwen3-VL-Flash and Qwen3-VL-Plus.
To assess the effectiveness of domain specific adaptation, we compare our framework with medical foundation models including HuluMed-4B and HuluMed-7B \cite{jiang2025hulu}. 
Furthermore, we evaluate small scale foundation models such as Qwen3-VL-2B, Qwen3-VL-4B, and Qwen3-VL-8B \cite{bai2025qwen3}. 
To isolate the contributions of our strategy, we include three internal variants: Qwen3-VL-2B (SFT), Qwen3-VL-4B (SFT) and Qwen3-VL-8B (SFT). 

\noindent\textbf{Evaluation Details.} 
We evaluate CogAlign on a comprehensive test suite comprising a total of 4,779 endoscopic samples across five distinct datasets. 
These benchmarks include CrohnIPI \cite{vallee2020crohnipi}, GastroVision \cite{jha2023gastrovision}, HyperKvasir \cite{borgli2020hyperkvasir}, Kvasir-Capsule \cite{smedsrud2021kvasir}, and The SEE-AI Project \cite{yokote2024small}.
Notably, The SEE-AI Project presents a significantly higher diagnostic challenge as it contains 235 multi-label samples, requiring the model to identify co-occurring pathologies simultaneously.
We report accuracy as the primary evaluation metric. 
For the multi-label cases, we employ a strict accuracy standard where a prediction is considered correct only if it exactly matches the complete set of ground truth pathologies.

\subsection{Main Results}

\noindent\textbf{Comparison with Large Foundation Models.} 
Despite their massive parameter scales, general-purpose MLLMs often struggle in specialized medical contexts. 
As illustrated in Tab.~\ref{tab:main_results}, proprietary models like Gemini 3 Pro and GPT-5 series achieve moderate performance but lack consistency. 
Qwen3-VL-Plus perform better, yet they still fall short in challenging scenarios like Kvasir-Capsule and The SEE-AI Project. 
In contrast, our CogAlign achieves a remarkable accuracy, surpassing Qwen3-VL-Plus by a significant margin.

\noindent\textbf{Comparison with Medical Foundation Models.} 
Specialized medical models do not exhibit a competitive edge. This underperformance can be attributed to their training paradigms, which often focus on general medical visual-question answering rather than the rigorous, fine-grained visual recognition required for gastrointestinal endoscopy. 
The clinical cognition alignment strategy effectively bridges this gap, ensuring the model attends to lesion features.

\begin{table}[t]
\centering
\scriptsize
\setlength{\tabcolsep}{13pt} 
\caption{
\textbf{Breakdown of Single-Label vs. Multi-Label diagnostic accuracy.}
We evaluate the ability to identify concurrent pathologies.
}
\vspace{-3mm}
\label{tab:single_multi_acc}
\begin{tabular}{lccc}
\toprule
\textbf{Model} & \textbf{Single-Label} & \textbf{Multi-Label} & \textbf{Average} \\
\midrule
\rowcolor{gray!15} \multicolumn{4}{l}{\textit{\textbf{Large Foundation Models}}} \\
Gemini 3 Flash & 21.68\% & 1.70\% & 20.69\% \\
Gemini 3 Pro & 26.06\% & 0.85\% & 24.82\% \\
GPT-5 Nano & 5.30\% & 0.43\% & 5.06\% \\
GPT-5 Mini & 10.43\% & 2.55\% & 10.04\% \\
GPT-5.2 & 11.69\% & 0.43\% & 11.13\% \\
Qwen3-VL-Flash & 38.34\% & 9.79\% & 36.93\% \\
Qwen3-VL-Plus & 42.76\% & 10.21\% & 41.16\% \\
\midrule
\rowcolor{gray!15} \multicolumn{4}{l}{\textit{\textbf{Medical Foundation Models}}} \\
Hulu-Med-4B & 8.12\% & 0.00\% & 7.72\% \\
Hulu-Med-7B & 9.02\% & 0.00\% & 8.58\% \\
\midrule
\rowcolor{gray!15} \multicolumn{4}{l}{\textit{\textbf{Small Foundation Models}}} \\
Qwen3-VL-2B & 16.81\% & 1.28\% & 16.05\% \\
Qwen3-VL-4B & 31.27\% & 5.96\% & 30.03\% \\
Qwen3-VL-8B & 36.77\% & 6.81\% & 35.30\% \\
\midrule
\rowcolor{gray!15} \multicolumn{4}{l}{\textit{\textbf{SFT on The Proposed Dataset}}} \\
Qwen3-VL-2B (SFT) & 56.91\% & 7.66\% & 54.49\% \\
Qwen3-VL-4B (SFT) & 64.66\% & 10.21\% & 61.98\% \\
Qwen3-VL-8B (SFT) & 69.19\% & 10.64\% & 66.31\% \\
\midrule
\rowcolor{green!15} \multicolumn{4}{l}{\textit{\textbf{Our Proposed Models}}} \\
\rowcolor{green!5}
CogAlign-2B & 59.93\% & 8.09\% & 57.38\% \\
\rowcolor{green!5}
CogAlign-4B & 66.81\% & 10.64\% & 64.05\% \\
\rowcolor{green!5}
CogAlign-8B & \textbf{70.47\%} & \textbf{13.62\%} & \textbf{67.67\%} \\
\bottomrule
\end{tabular}
\vspace{-6mm}
\end{table}

\subsection{Multi-Label Disease Diagnosis}
In real-world clinical environments, patients frequently present with concurrent gastrointestinal pathologies, requiring models to identify multiple co-occurring conditions rather than a single dominant lesion. 
As shown in Tab.~\ref{tab:single_multi_acc}, general foundation models struggle significantly in this setting, often exhibiting tunnel vision where secondary pathologies are ignored; for instance, specialized medical models like Hulu-Med-7B completely fail to detect multi-label cases. 
In contrast, our CogAlign framework demonstrates superior performance. 
This improvement confirms that our hierarchical reasoning chain and counterfactual-driven reinforcement learning effectively force the model to conduct a comprehensive scan of the mucosal surface rather than fixating on spurious or singular features.

\begin{figure}[t]
    \centering
    \includegraphics[width=\linewidth]{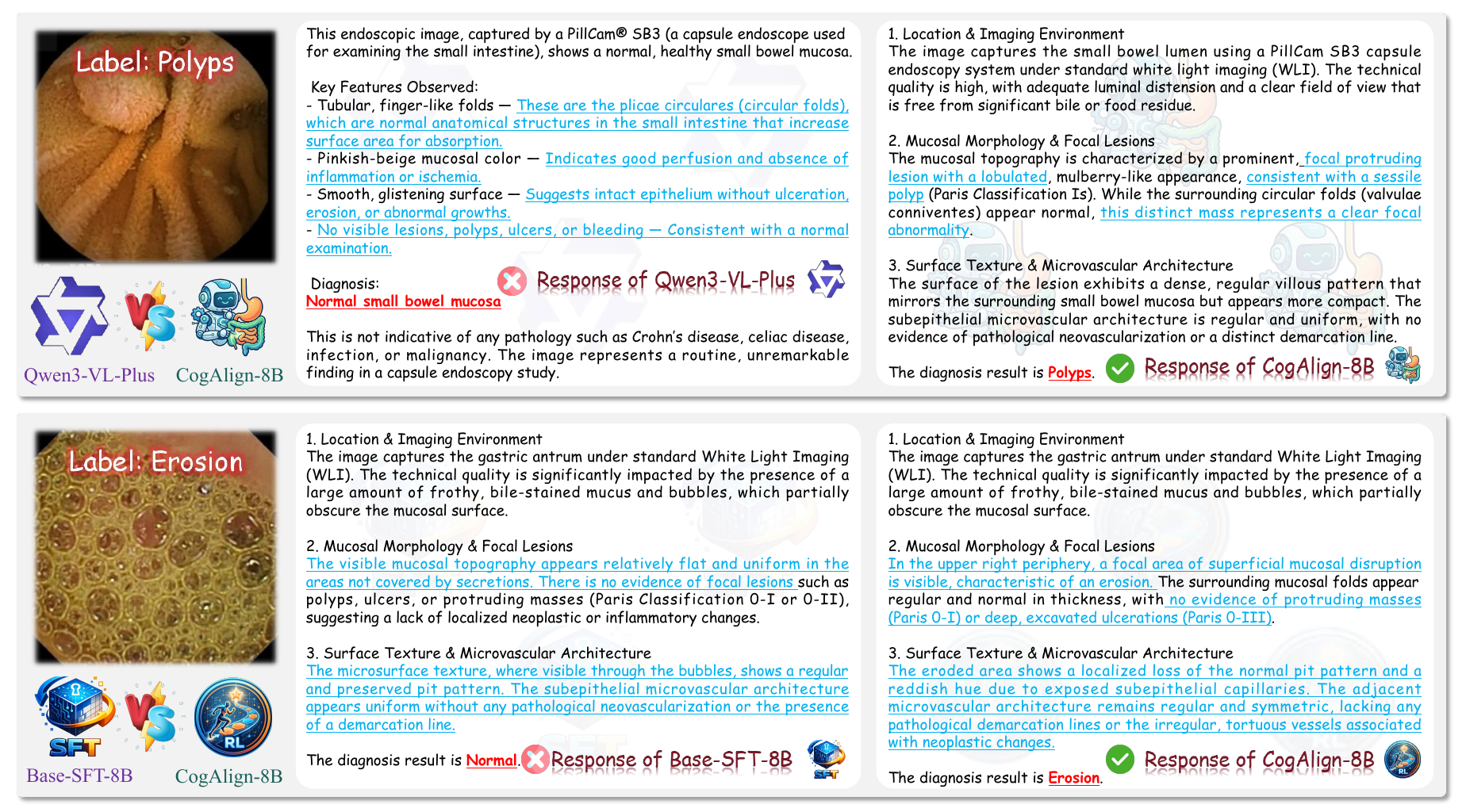}
    \vspace{-6mm}
    \caption{\textbf{Case study between CogAlign and baseline models.} The top row demonstrates CogAlign's ability to detect a subtle polyp via hierarchical clinical cognition. 
    The bottom row highlights CogAlign's robustness to visual noise in identifying erosion.}
    \label{fig:case}
    \vspace{-6mm}
\end{figure}

\subsection{Case Study}
To provide a qualitative evaluation of our proposed approach, we present a comparative case study in Fig. \ref{fig:case}. 
The top row illustrates the superiority of our framework over the general foundation model Qwen3 VL Plus. 
In this scenario, the endoscopic image contains a subtle polyp. 
The general model fails to identify the lesion and incorrectly predicts a normal mucosa. 
Conversely, our model leverages the internalized clinical cognitive pathway to systematically analyze the image. 
By sequentially evaluating the anatomical location, mucosal morphology, and microscopic details, our model accurately detects the lobulated protruding lesion and correctly concludes the diagnosis as polyps. 

The bottom row highlights the effectiveness of our counterfactual driven reinforcement learning stage by comparing the full pipeline against the Base-SFT-8B variant. 
The input image is heavily obscured by environmental noise, specifically frothy bile stained mucus and bubbles. 
The Base-SFT-8B model, lacking causal diagnostic grounding, is misled by these environmental artifacts and hallucinates a normal diagnosis. 
In contrast, our fully trained model successfully ignores the spurious visual noise. 
Guided by the causal alignment phase, it focuses precisely on the superficial mucosal disruption and accurately identifies the erosion.

\begin{figure}[t]
    \centering
    \includegraphics[width=\linewidth]{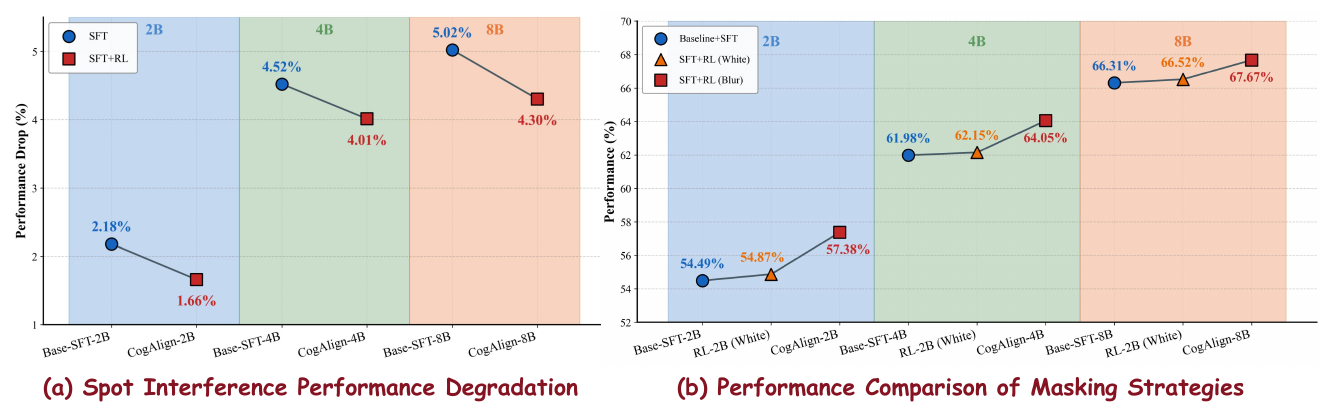}
    \vspace{-6mm}
    \caption{\textbf{Detailed analysis of model robustness and masking strategies.} (a) Performance degradation under spot interference. (b) Comparison of masking techniques.}
    \label{fig:analysis}
    \vspace{-6mm}
\end{figure}

\subsection{Robustness Analysis}
To evaluate the resilience of our proposed framework against environmental interference, we conduct a robustness analysis by applying simulated spot interference to the test images. 
This technique explicitly simulates the mucosal bubbles and specular reflections that frequently corrupt clinical endoscopic observations
As illustrated in Fig. \ref{fig:analysis}(a), the baseline models fine tuned only with SFT suffer a severe degradation in diagnostic accuracy when exposed to visual perturbations. 
This vulnerability indicates that standard training paradigms overfit to spurious background correlations. 
In contrast, the complete CogAlign framework exhibits remarkable stability, maintaining high performance across all model scales. 

\subsection{Selection of Masking Strategy}
The generation of counterfactual normal samples requires obliterating pathological evidence while preserving the surrounding contextual environment. 
We investigate the impact of different erasure techniques by comparing solid white masking against high intensity Gaussian blurring. 
As depicted in Fig. \ref{fig:analysis}(b), employing a Gaussian blur to synthesize counterfactuals yields consistently higher diagnostic accuracy compared to utilizing solid white patches. 
We attribute this performance discrepancy to the naturalness of the modified images. 
Solid white masks introduce sharp artificial boundaries and out of distribution visual signals that can destabilize the reinforcement learning optimization process. 
Conversely, Gaussian blurring effectively neutralizes the diagnostic features while maintaining a smooth visual texture, thereby providing a reliable reference for causal rectification and enabling the model to accurately isolate lesion representations.

\subsection{Ablation Study}

\noindent\textbf{Effect of Clinical Cognition Alignment.} 
To validate the necessity of bridging the gap between general reasoning and standardized clinical protocols, we compare the performance of the vanilla foundation models against those fine-tuned on our hierarchical clinical cognition dataset. 
As observed in Fig.~\ref{fig:ablation}, applying our clinical cognition alignment via SFT dramatically significantly boosts this performance. 
This substantial improvement confirms that explicitly internalizing the expert cognitive flow is essential for unlocking the potential of MLLMs.

\noindent\textbf{Effect of Clinical Cognition Reward.} 
To assess the impact of rewards, we conduct an ablation study by removing the Clinical Cognition Reward $R_{cog}$ from the full reward schema. 
As shown in Fig.~\ref{fig:ablation}, removing $R_{cog}$ leads to a noticeable degradation in performance.
Specifically, in the absence of constraints on semantic clinical features, the model's intermediate reasoning degrades into vague descriptions that lack genuine visual-pathological grounding.

\begin{figure}[t]
    \centering
    \includegraphics[width=\linewidth]{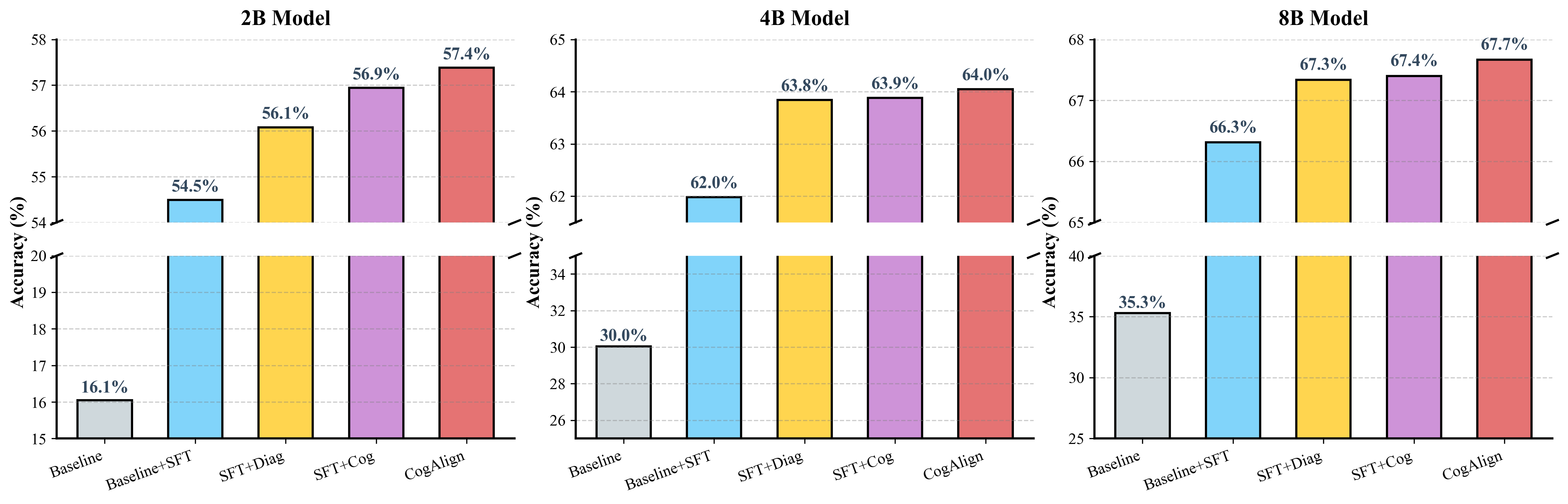}
    \vspace{-6mm}
    \caption{\textbf{Ablation study of individual modules in the CogAlign.} We examine the contribution of SFT and the proposed reinforcement learning rewards.}
    \label{fig:ablation}
    \vspace{-6mm}
\end{figure}

\noindent\textbf{Effect of Diagnostic Consistency Reward.} 
We further evaluate the contribution of the Diagnostic Consistency Reward $R_{diag}$. 
By excluding $R_{diag}$ and relying solely on the format and cognition rewards, the model focuses heavily on generating descriptive text but occasionally fails to draw the correct conclusion from its own analysis. 
Experimental results in Fig.~\ref{fig:ablation} indicate that removing this reward causes a significant decline in diagnostic accuracy. 

\section{Conclusion}
In this paper, we proposed CogAlign, a novel framework designed to bridge the cognitive gap between general MLLMs and the rigorous standards of gastrointestinal diagnosis. 
Addressing the critical challenges of clinical cognitive misalignment and causal disconnect, we introduced a systematic clinical cognition alignment strategy. 
First, we constructed a hierarchical clinical cognition dataset and employed SFT to internalize expert-level diagnostic logic, compelling the model to strictly follow a trajectory from anatomical localization and morphological evaluation to micro-detail analysis. 
Second, guided by our theoretical analysis on shortcut convergence, we implemented a counterfactual-driven GRPO strategy. 
By utilizing counterfactual normal samples and clinical-cognition-centric rewards, we enforced causal rectification, ensuring diagnoses are grounded in pathological lesion features. 
Extensive experiments across five diverse benchmarks demonstrate that CogAlign establishes a new SoTA, significantly enhancing diagnostic performance in complex clinical scenarios.



%
%
\bibliographystyle{splncs04}
\bibliography{main}

\clearpage

\end{document}